\documentclass[letterpaper, 10 pt, conference]{ieeeconf}  

\IEEEoverridecommandlockouts  
                                                          
\usepackage[T1]{fontenc}
\usepackage{graphicx}
\usepackage{amsmath, amssymb}
\usepackage{xcolor}
\usepackage{bm}

\graphicspath{ {Figures/} }

\newlength{\mylen}
\renewcommand{\baselinestretch}{0.96}

\newtheorem{thm}{Theorem}

\DeclareMathOperator{\ubf}{\mathbf{u}}
\DeclareMathOperator{\x}{\mathbf{x}}
\DeclareMathOperator{\y}{\mathbf{y}}
\DeclareMathOperator{\yh}{\hat{\mathbf{y}}}

\DeclareMathOperator{\z}{\mathbf{z}}
\DeclareMathOperator{\vbf}{\mathbf{v}}
\DeclareMathOperator{\etabf}{\bm{\eta}} 

\DeclareMathOperator{\m}{\bm{\mu}} 
\DeclareMathOperator{\GP}{\mathcal{GP}}
\DeclareMathOperator{\No}{\mathcal{N}} 
\DeclareMathOperator{\zerobf}{\mathbf{0}} 
\DeclareMathOperator{\B}{\mathcal{B}} 
\DeclareMathOperator{\KL}{\mathcal{D}_{KL}} 

\DeclareMathOperator{\V}{\mathcal{V}} 
\DeclareMathOperator{\D}{\mathcal{D}} 
\DeclareMathOperator{\R}{\mathbb{R}} 
\DeclareMathOperator{\sk}{\textup Skew}

\DeclareMathOperator{\diag}{\textup diag}

\DeclareMathOperator{\Vbf}{V}

\DeclareMathOperator{\e}{\mathbf{e}} 
\DeclareMathOperator{\p}{\mathbf{p}} 
\DeclareMathOperator{\E}{\mathbb{E}}

\title{\LARGE \bf Stability Enhanced Gaussian Process Variational Autoencoders}

\author{Carl R. Richardson$^{*}$ \hspace{3mm} Jichen Zhang$^{*}$ \hspace{3mm} Ethan King \hspace{3mm} Ján Drgoňa
\thanks{CRR, JZ are with Department of Engineering Science, University of Oxford, Parks Road, Oxford, OX1 3PJ, UK
        {\tt\footnotesize  \{carl.richardson, jichen.zhang\}@eng.ox.ac.uk}}%
\thanks{EK is with Pacific Northwest National Laboratory, Richland, Washington, 99354, USA
        {\tt\footnotesize ethan.king@pnnl.gov}}%
\thanks{JD is with Johns Hopkins University, Baltimore, Maryland, 21218, USA
        {\tt\footnotesize jdrgona1@jh.edu}. $^*$Equal contribution.}%
\thanks{This research was supported by the UKRI centre for doctoral training in Machine Intelligence for Nano-electronic Devices and Systems [EP/S024298/1] and the AT SCALE initiative via the Laboratory Directed Research and Development (LDRD) investments at Pacific Northwest National Laboratory (PNNL). PNNL is a national laboratory operated for the U.S. Department of Energy (DOE) by Battelle Memorial Institute under Contract No. DE-AC05-76RL0-1830.}%
}

\begin{document}

\maketitle
\thispagestyle{empty}
\pagestyle{empty}

\begin{abstract}
A novel stability-enhanced Gaussian process variational autoencoder (SEGP-VAE) is proposed for indirectly training a low-dimensional linear time invariant (LTI) system, using high-dimensional video data. The mean and covariance function of the novel SEGP prior are derived from the definition of an LTI system, enabling the SEGP to capture the indirectly observed latent process using a combined probabilistic and interpretable physical model. The search space of LTI parameters is restricted to the set of semi-contracting systems via a complete and unconstrained parametrisation. As a result, the SEGP-VAE can be trained using unconstrained optimisation algorithms. Furthermore, this parametrisation prevents numerical issues caused by the presence of a non-Hurwitz state matrix. A case study applies SEGP-VAE to a dataset containing videos of spiralling particles. This highlights the benefits of the approach and the application-specific design choices that enabled accurate latent state predictions.
\end{abstract}

\section{Introduction} \label{sec:intro}

Variational autoencoders (VAEs) are successful unsupervised models for learning compact representations of high-dimensional data \cite{doersch2016tutorial}, with many applications including image classification \cite{gregor2015draw}, long-horizon prediction \cite{saxena2021clockwork}, mobile robotics \cite{nagano2022spatio}, and autonomous-driving \cite{amini2018variational}. In time-varying applications, the latent process typically obeys physical rules. Model-driven approaches exploit these rules, but such knowledge is often incomplete, whereas purely data-driven approaches require little prior information but can be data and computationally expensive whilst lacking interpretability.

Koopman operator theory provides a framework for representing nonlinear dynamical systems through an infinite-dimensional linear operator acting on observable functions of the state \cite{budivsic2012applied}, enabling computational approximation and analysis of otherwise intractable systems \cite{lian2020gaussian, miao2026learning}. Although theoretically infinite-dimensional, finite-dimensional approximations are often sufficient in practice, motivating VAEs with linear latent dynamics \cite{beckers2023physics} as approximations of nonlinear latent processes.

We consider videos generated by an unobserved low-dimensional dynamical process and seek to learn this process directly from the video data. Such settings arise in materials science, for example, where thin-film deposition cannot be measured directly but related diffraction patterns can be recorded \cite{kaspar2025machine}. Related approaches include state-space models \cite{fraccaro2017disentangled} and Gaussian process (GP) based VAEs \cite{beckers2023physics, pearce2020gaussian}. More broadly, physics-informed learning seeks to improve the robustness of machine-learning models for scientific applications by verifying properties such as stability after training \cite{Richardson2023strengthened, richardson2026analysis}, embedding prior knowledge into model architectures \cite{bronstein2021geometric, richardson2025lurie}, or encouraging desirable behaviour through loss-function design \cite{degrave2022magnetic}.

Encoding prior knowledge reduces model generality but can improve performance; when the imposed structure is ubiquitous, broad applicability can be retained while benefiting from physics-informed learning \cite{Drgona2025}. Stability is one such property, appearing across physical \cite{khalil2002nonlinear}, biological \cite{kozachkov2020achieving}, and neural systems \cite{Drgona2022}. Semi-contraction is a global stability property \cite{lohmiller1998contraction} under which the distance between any two trajectories never increases, irrespective of their initial conditions. It provides a robust notion of stability and applies to time-varying systems, including those with external inputs.

\textbf{Contribution}: We propose a novel stability-enhanced Gaussian process (SEGP), whose mean and covariance functions are derived from a linear time-invariant (LTI) system. Semi-contraction is embedded in the GP through an unconstrained and complete parametrisation: any LTI semi-contracting system can be constructed from the unconstrained variables. The parametrisation also inherently avoids numerical issues during training associated with a non-Hurwitz $A$ matrix. Incorporating the SEGP into an existing GP-VAE framework yields the SEGP-VAE, extending the physics-enhanced GP-VAE (PEGP-VAE) \cite{beckers2023physics} by requiring less prior knowledge of the latent process and accommodating unknown initial conditions. When trained on videos of a particle spiralling in a plane, the SEGP-VAE recovers the indirectly observed latent process through a probabilistic yet interpretable physical model which can be used for downstream tasks such as control design. 

\textbf{Paper structure}: The theoretical contribution is separated from application-specific design choices. Background and the problem setup are presented in Sections \ref{sec:pre} and \ref{sub:ps}, respectively. Section \ref{sec:SEGP} introduces the SEGP, followed by the SEGP-VAE in Section \ref{sub:GPVAE}. Finally, a case study is presented in Section \ref{sec:app} which includes the data, application-specific choices such as the encoder architecture, and empirical results; these design choices may differ for other applications.

\section{Preliminaries} \label{sec:pre}

\subsection{Notation} \label{sub:not}

The set of real numbers greater than or equal to zero is denoted by $\R_{\geq 0}$. The following sets of $j \times j$ matrices: symmetric (diagonal) positive definite, skew-symmetric and lower triangular with non-negative (positive) diagonal entries, are respectively denoted by  $\mathbb{S}_{+}^{j} (\mathbb{D}_{+}^{j}), \sk(j)$ and $\mathbb{L}_{\geq 0}^{j}$ ($\mathbb{L}_{+}^{j}$). The matrix $H$ with elements $h_{jq}$ is denoted by $H = [h_{jq}]$. A negative (semi-) definite matrix inequality is denoted by $H \prec 0$ ($H \preceq 0$). A vector function $\z(t) \in \R^{m}$ at a set of discrete time points, $T = \{t_{j}\}_{j=1}^{N}$, is denoted by $\z(T) = [ z_{1}(t_{1}) \hspace{1mm} \dots \hspace{1mm} z_{1}(t_{N}) \hspace{1mm} \dots \hspace{1mm} z_{m}(t_{1}) \hspace{1mm} \dots \hspace{1mm} z_{m}(t_{N})] \in \R^{mN}$. Similarly, a matrix function $H(t, t') =  [h_{jq}(t,t')] \in \R^{m \times m}$ over $T$ and $T_{*} = \{\tau_{q} \}_{q=1}^{N_{*}}$ is denoted by
\begin{equation*}
H(T,T_{*}) = \hspace{-1mm}
\setlength\arraycolsep{3pt}
\begin{bmatrix}
H_{11}(T, T_{*}) & \dots & H_{1m}(T, T_{*}) \\
\vdots & \ddots & \vdots \\
H_{m1}(T, T_{*}) & \dots & H_{mm}(T, T_{*})
\end{bmatrix}
\hspace{-1mm} \in \R^{mN \times mN_{*}}
\end{equation*}
where $H_{il}(T,T_{*}) = [h_{il}(t_{j}, \tau_{q})] \in \R^{N \times N_{*}}$.

An $m$-dimensional vector-valued GP is denoted by $\z(t) \sim \GP(\m_{z}(t), K_{z}(t,t'))$ where $\m_{z}(t) \in \R^{m}$ is the vector-valued mean function  and $K_{z}(t,t') \in \R^{m \times m}$ is the matrix-valued covariance function. The associated Gaussian distribution over $T$ is denoted by $\z(T) \sim \No \big(\m_{z}(T), K_{z}(T,T) \big)$ where $\m_{z}(T) \in \R^{mN}$ and $K_{z}(T,T) \in \R^{mN \times mN}$ denote the mean vector and covariance matrix over $T$, respectively.

\subsection{Gaussian Processes} \label{sub:GP}

This section provides a condensed background on GPs for vector-valued functions; refer to \cite{williams2006gaussian, alvarez2012kernels} for a more comprehensive introduction. Two critical assumptions are made in the definition of a GP: (i) all variables can be represented as a sample from a multivariate Gaussian; (ii) any subset of variables can also be represented as a sample from a multivariate Gaussian. Suppose that i.i.d Gaussian noise is defined by $\etabf(t) \sim \GP(\zerobf, \Sigma(t,t'))$, with $\Sigma(t,t') = \diag(\sigma_{1}^{2}, \dots, \sigma_{m}^{2})$, and the latent state is governed by $\y(t) \sim \GP(\m_{y}(t), K_{y}(t,t'))$, then the corrupted latent state, $\yh(t) = \y(t) + \etabf(t)$, is $\yh(t) \sim \GP(\m_{y}(t), K_{y}(t,t') + \Sigma(t,t'))$. By definition, the joint Gaussian of the corrupted latent state (observed at $T$) and the latent state (at test points $T_{*}$) is
\begin{equation} \label{eq:GP_joint}
\begin{bmatrix}
\yh(T) \\
\y(T_{*})
\end{bmatrix}
\sim \No \Biggl(
\begin{bmatrix}
\m_{y}(T) \\
\m_{y}(T_{*})
\end{bmatrix}
,
\begin{bmatrix}
K_{y}(T, T) + \Sigma & K_{y}(T, T_{*}) \\
K_{y}(T_{*}, T) & K_{y}(T_{*}, T_{*})
\end{bmatrix}
\Biggl)
\end{equation}
where $\Sigma := \Sigma(T,T)$ and $K_{y}(T, T_{*}) = K_{y}(T_{*}, T)^{\top}$.

The choice of mean and covariance functions impose a modelling bias or an opportunity to encode prior knowledge into the model. For example, the popular \emph{squared exponential} covariance function implies that the covariance between observations exponentially decays as a function of their separation in time. Refer to \cite{alvarez2012kernels} for a comparison of different covariance functions. 

Assume the existence of a dataset $\D_{\hat{y}} := (T, \hat{\mathcal{Y}})$, where $\hat{\mathcal{Y}} := \{ \hat{\y}_{j}(T)\}_{j=1}^{N_{\hat{y}}}$ is a batch of corrupted latent state trajectories. In this setting, $\y(t) \sim \GP(\m_{y}(t), K_{y}(t,t'))$ is referred to as the GP prior of the latent state. By Bayes' Theorem, the predictive distribution of the latent state at the unobserved $T_{*}$ can be directly extracted from (\ref{eq:GP_joint}). The predictive distribution is given by (\ref{eq:GP_pred})-(\ref{eq:GP_pred_cov}) and provides a distribution conditioned on the dataset. The predictive distribution is referred to as the posterior when $T_{*} = T$.
\begin{equation} \label{eq:GP_pred}
\y(T_{*}) \vert T_{*}, \D_{\hat{y}} \sim \No \big( \m_{y | \hat{y}}(T_{*}), K_{y | \hat{y}}(T_{*}, T_{*}) \big)
\end{equation}  
\vspace{-3\mylen}
\begin{multline} \label{eq:GP_pred_mean}
\m_{y | \hat{y}}(T_{*}) := \m_{y}(T_{*}) + \\ K_{y}(T_{*}, T) \big( K_{y}(T, T) + \Sigma \big)^{-1} \big( \hat{\y}(T) - \m_{y}(T) \big)
\end{multline}
\vspace{-3\mylen}
\begin{multline} \label{eq:GP_pred_cov}
K_{y | \hat{y}}(T_{*}, T_{*}) := K_{y}(T_{*}, T_{*}) - \\ K_{y}(T_{*}, T) \big( K_{y}(T, T) + \Sigma \big)^{-1} K_{y}(T, T_{*})
\end{multline}

\subsection{Latent Force Models} \label{sub:LFM}

Latent force models (LFM) were proposed in \cite{alvarez2013linear} for incorporating differential equations into GPs. The LFM framework can be applied to LTI systems \eqref{eq:LTI} by noting the existence of an analytical solution of the form  (\ref{eq:LFM}) where $G: \R_{+} \times \R_{+} \rightarrow \R^{n \times p}$ is known as the \emph{Green's function}. Since a GP is closed under linear operators \cite{williams2006gaussian}, placing a GP on $\ubf(t) \sim \GP(\m_{u}(t), K_{u}(t, t'))$ and assuming $\x(0) \sim \No(\m_{x0}, \Sigma_{x0})$ results in a GP $\x(t) \sim \GP(\m_{x}(t), K_{x}(t, t'))$ where $\m_{x}(t)$ and $K_{x}(t, t')$ are determined by (\ref{eq:LFM}). Furthermore, since $\y$ is a linear combination of $\x$ and $\ubf$, it is straightforward to determine $\y(t) \sim \GP(\m_{y}(t), K_{y}(t, t'))$.
\begin{equation} \label{eq:LFM}
\x(t) = \e^{At} \x(0) + \int_{0}^{t} \underbrace{\e^{A(t - \tilde{t})}B}_{:= G(t, \tilde{t})} \ubf( \tilde{t}) d \tilde{t}
\end{equation} 

\section{Problem Setup} \label{sub:ps}

We consider the problem of modelling an unobserved latent process from a batch of related video recordings, $\V := \{\vbf_{j}(T)\}_{j=1}^{N_{V}}$, where $T := \{t_{i}\}_{i=1}^{N}$. The pixel values of the square $d$-dimensional frames, over $T$, are denoted by $\vbf_{j}(T) = [v_{11}(t_{1}) \hspace{1mm} \dots \hspace{1mm} v_{11}(t_{N}) \hspace{1mm} \dots \hspace{1mm} v_{dd}(t_{1}) \hspace{1mm} \dots \hspace{1mm} v_{dd}(t_{N})]$. As the latent states are not measured, this is an unsupervised learning problem with dataset $\D := (T, \V)$. In this work, the latent process is assumed to be an LTI system defined by (\ref{eq:LTI}) with parameters $A \in \R^{n \times n}, B \in \R^{n \times p}, C \in \R^{m \times n}$, $D \in \R^{m \times p}$, and initial condition $\x(0) \sim \No(\m_{x0}, \Sigma_{x0})$.
\begin{subequations} \label{eq:LTI}
\begin{align}
\dot{\x}(t) &= A \x(t) + B \ubf(t) \label{eq:LTIx} \\
\y(t) &= C \x(t) + D \ubf(t) \label{eq:LTIy}
\end{align}
\end{subequations}
Several mild assumptions are made, these are: (i) $\ubf(t) \sim \GP(\m_{u}(t), K_{u}(t, t'))$ with known mean and covariance functions; (ii) the LTI parameters form a semi-contracting system (Theorem \ref{th:scon}); (iii) the latent state is corrupted by i.i.d Gaussian noise; (iv) each frame of the video is only dependent upon the corrupted latent state, at the corresponding time.

\section{Stability Enhanced Gaussian Process} \label{sec:SEGP}
This section derives the mean and covariance functions of the SEGP. This is later included within the SEGP-VAE, as in Fig. \ref{fig:GP_VAE}, to model the prior over the latent process. Theorem \ref{th:LTI_GP} derives the mean and covariance functions of a general LTI system (\ref{eq:LTI}). The matrices $A, B, C, D$ are parameters of the GP which must be trained. Theorem \ref{th:scon_param} provides an unconstrained and complete parametrisation of the $A$ matrix such that only semi-contracting LTI models can be learnt. Coupling Theorem \ref{th:LTI_GP} with the parametrisation in Theorem \ref{th:scon_param} results in the SEGP.

\begin{thm} \label{th:LTI_GP}
Given $\ubf(t) \sim \GP \big( \m_{u}(t), K_{u}(t,t') \big)$ and $\x(0) \sim \No (\m_{x0}, \Sigma_{x0})$,  the latent state governed by (\ref{eq:LTI}) has the prior distribution $\y(t) \sim \GP(\m_{y}(t), K_{y}(t,t'))$ with the following covariance and mean functions
\begin{multline} \label{eq:LTI_kernel}
K_{y}(t, t') = C \e^{At} \Sigma_{x0} \e^{At'\top} C^{\top} + D K_{u}(t, t') D^{\top} \\ 
+ \int_{0}^{t} \int_{0}^{t'} G_{C}(t, \tilde{t}) K_{u}(\tilde{t}, \hat{t})  G_{C}(t', \hat{t})^{\top} d \hat{t} d \tilde{t} \\
+ \int_{0}^{t} G_{C}(t, \tilde{t}) K_{u}(\tilde{t}, t') D^{\top} d \tilde{t} \\
+ \int_{0}^{t'} D K_{u}(t, \hat{t}) G_{C}(t', \hat{t})^{\top} d \hat{t}
\end{multline}
\begin{equation} \label{eq:LTI_mean}
\vspace{-\mylen}
\m_{y}(t) = C \e^{At} \m_{x0} + \int_{0}^{t} G_{C}(t,\tilde{t})\m_{u}(\tilde{t}) d \tilde{t} + D \m_{u}(t)
\end{equation}
where $G_{C}(\tilde{t}, \hat{t}) := C \e^{A(\tilde{t} - \hat{t})} B$ is the Green's function of (\ref{eq:LTI}) multiplied by $C$.
\end{thm}

\emph{Proof}: The solution to (\ref{eq:LTIx}) is given by (\ref{eq:LFM}). Substituting this into (\ref{eq:LTIy})  gives
\begin{equation} \label{eq:y_sol}
\y(t) = C \e^{At} \x(0) + \int_{0}^{t} G_{C}(t, \tilde{t}) \ubf (\tilde{t}) d \tilde{t} + D \ubf(t)
\end{equation}
The LFM framework, summarised in Section \ref{sub:LFM}, highlights there must be a GP over $\y(t)$. The mean is derived as follows
\begin{equation*}
\begin{split}
\m_{y}(t) 
&= C \e^{At} \E[\x(0)] + \int_{0}^{t} \hspace{-2mm} G_{C}(t, \tilde{t}) \E [\ubf (\tilde{t})] d \tilde{t} + D \E [\ubf(t)]
\end{split}
\end{equation*}
Substituting the mean of $\x(0)$ and the mean function of $\ubf(\cdot)$ results in (\ref{eq:LTI_mean}). The covariance function between $\y(t)$ and $\y(t')$ is defined by
\begin{equation*}
K_{y}(t, t') = \E \big[ \big( \y(t) - \m_{y}(t) \big) \big( \y(t') - \m_{y}(t') \big)^{\top} \big]
\end{equation*}
Substituting (\ref{eq:LTI_mean}) and (\ref{eq:y_sol}) leads to
\begin{equation*}
K_{y}(t, t') = \E[Y_{1}Y_{3}^{\top}] + \E[Y_{1}Y_{4}^{\top}] + \E[Y_{2}Y_{3}^{\top}] + \E[Y_{2}Y_{4}^{\top}]
\end{equation*}
\begin{equation*}
\begin{split}
Y_{1} &:= C \e^{At} \big(\x(0) - \m_{x0} \big) \\
Y_{2} &:= \hspace{-1mm} \int_{0}^{t} \hspace{-2mm} G_{C}(t, \tilde{t}) \big( \ubf(\tilde{t}) - \m_{u}(\tilde{t}) \big) d \tilde{t} + D \big( \ubf(t) - \m_{u}(t) \big) \\
Y_{3} &:= C \e^{At'} \big(\x(0) - \m_{x0} \big) \\
Y_{4} &:= \hspace{-1mm} \int_{0}^{t'} \hspace{-3mm} G_{C}(t', \hat{t}) \big( \ubf(\hat{t}) - \m_{u}(\hat{t}) \big) d \hat{t} + D \big( \ubf(t') - \m_{u}(t') \big)
\end{split}
\end{equation*}
We now focus on each term individually.
\begin{equation*}
\E[Y_{1}Y_{3}^{\top}] = C \e^{At} \Sigma_{x0} \e^{At'\top} C^{\top}
\end{equation*}
\begin{multline*}
\E[Y_{1}Y_{4}^{\top}] = C \e^{At} \times \\ 
\int_{0}^{t'} \underbrace{\E \big[ \big(\x(0) - \m_{x0} \big) \big( \ubf(\hat{t}) - \m_{u}(\hat{t}) \big)^{\top} \big]}_{\text{covariance between $\x(0)$ and $\ubf(\hat{t})$}} G_{C}(t', \hat{t})^{\top} d \hat{t} \\ 
+ C \e^{At} \underbrace{\E \big[ \big(\x(0) - \m_{x0} \big) \big( \ubf(t') - \m_{u}(t') \big)^{\top} \big]}_{\text{covariance between $\x(0)$ and $\ubf(t')$}} D^{\top} = 0 
\end{multline*} 
as $\x(0)$ and $\ubf(\cdot)$ are independent. By the same logic
\begin{equation*}
\E[Y_{2}Y_{3}^{\top}] = 0
\end{equation*}
\begin{multline*}
\E[Y_{2}Y_{4}^{\top}] = \int_{0}^{t} \int_{0}^{t'} G_{C}(t, \tilde{t}) K_{u}(\tilde{t}, \hat{t}) G_{C}(t', \hat{t})^{\top} d \hat{t} d \tilde{t}  \\ 
+ \int_{0}^{t} G_{C}(t, \tilde{t}) K_{u}(\tilde{t}, t') D^{\top} d \tilde{t} \\
+ \int_{0}^{t'} K_{u}(t, \hat{t}) G_{C}(t', \hat{t})^{\top} d \hat{t} \\
+ D K_{u}(t, t') D^{\top}
\end{multline*}
Substituting each term into $K_{y}(t,t')$ results in (\ref{eq:LTI_kernel}). $\hfill \square$

In \cite{beckers2023physics} the authors presented a covariance function for (\ref{eq:LTI}) when $D=0$, $\x(0)=\zerobf$ and $\ubf(t) \sim \GP \big(\zerobf, K_{u}(t, t') \big)$. We extend this to widen the scope of applicability to account for the cases when $D \neq 0$ (e.g., analogue circuits), $\x(0) \sim \No (\m_{x0}, \Sigma_{x0})$, such as aircraft tracking, and an input signal with non-zero mean function (e.g., stochastic interest rates).

Directly applying Theorem \ref{th:LTI_GP} is problematic because non-Hurwitz $A$ matrices are permitted, which destabilise training due to unbounded mean and covariance functions. Assumption (ii) in Section \ref{sub:ps} address this issue; however, assuming the LTI parameters form a semi-contracting system requires the following well established result to hold. 
\begin{thm}[Semi-contracting LTI system] \label{th:scon}
The LTI system (\ref{eq:LTI}) is semi-contracting if and only if there exists $P \in \mathbb{S}^{n}_{+}$ which satisfies the following linear matrix inequality (LMI)
\begin{equation} \label{eq:scon}
P A + A^{\top} P \preceq 0
\end{equation}
\end{thm}

\emph{Proof}: Refer to the appendix. \hfill $\square$

The parameter space of (\ref{eq:LTI}) which satisfies Theorem \ref{th:scon} is
\begin{equation} \label{eq:Om}
\Omega := \{(A,B,C,D) \hspace{1mm} \vert \hspace{1mm} \exists P \in \mathbb{S}_{+}^{n}: P A + A^{\top} P \preceq 0 \}
\end{equation}
where $B, C, D$ are unconstrained. The following result provides a complete parametrisation of the set $\Omega$ in terms of unconstrained variables, allowing the SEGP to be implemented using parametrised layers and trained using unconstrained optimisers, like gradient descent.

\begin{figure*}[t!]
\centering
\includegraphics[width=0.8\textwidth]{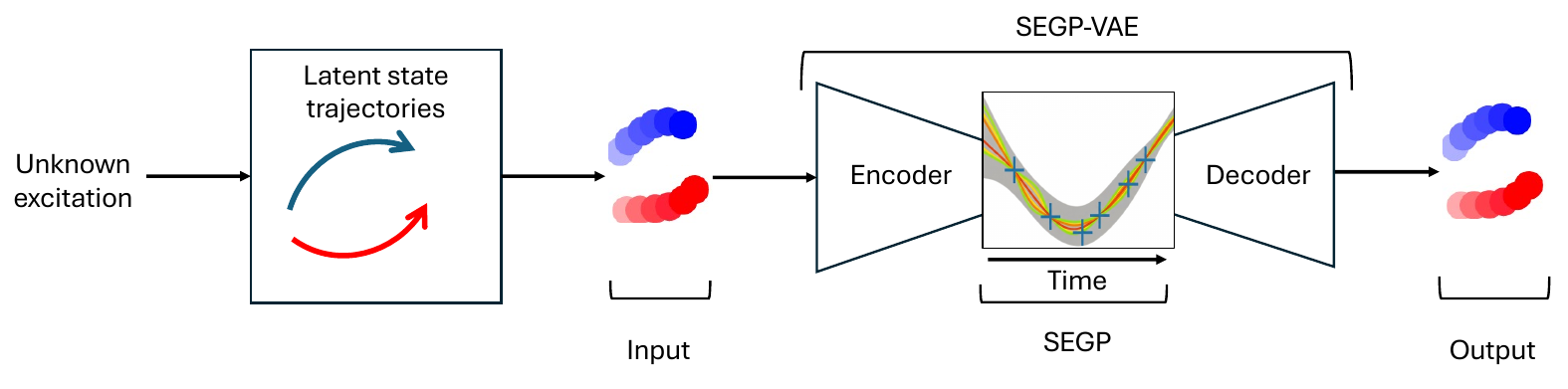}
\vspace{-2\mylen}
\caption{SEGP-VAE architecture with semi-contracting kernel.}
\label{fig:GP_VAE}
\vspace{-2\mylen}
\end{figure*}

\begin{thm} \label{th:scon_param}
$(A, B, C, D) \in \Omega$ if and only if there exists $\Vbf_{1} \in \mathbb{L}_{+}^{n}$, $\Vbf_{2} \in \mathbb{L}_{\geq 0}^{n}$ and $\Vbf_{3} \in \sk(n)$ related to $A$  by
\begin{align} \label{eq:scon_param}
A &= -\frac{1}{2} P^{-1}\Vbf_{2} \Vbf_{2}^{\top} + P^{-1} \Vbf_{3} &
P &= \Vbf_{1} \Vbf_{1}^{\top}
\end{align}
with $B \in \R^{n \times p}, C \in \R^{m \times n}$ and $D \in \R^{m \times p}$.
\end{thm}

\emph{Proof (Necessity)}: If $(A,B, C, D) \in \Omega$, there exists $P \in \mathbb{S}_{+}^{n}$ such that \eqref{eq:scon} holds. As the left hand side is symmetric and negative semi-definite, a Cholesky decomposition must exist such that
$PA + A^{\top}P = - \Vbf_{2} \Vbf_{2}^{\top}$. This equality is satisfied if and only if $PA = - \frac{1}{2} \Vbf_{2} \Vbf_{2}^{\top} + \Vbf_{3}$ where $\Vbf_{3} \in \sk(n)$. Reconstructing $A$ results in (\ref{eq:scon_param}).

\emph{(Sufficiency)}: If there exists $\Vbf_{1}$ as stated in Theorem \ref{th:scon_param}, it immediately follows that $P$ is positive definite, symmetric and full rank. Next, constructing $A$ as in (\ref{eq:scon_param}) gives
\begin{align*}
PA + A^{\top}P 
&= -\Vbf_{2} \Vbf_{2}^{\top} + \Vbf_{3} + \Vbf_{3}^{\top}
&= -\Vbf_{2} \Vbf_{2}^{\top}
\end{align*}
which is negative semi-definite. $\hfill \square$

As specified in Theorem \ref{th:scon_param}, the SEGP has trainable parameters $\psi := \{\Vbf_{1}, \Vbf_{2}, \Vbf_{3}, B, C, D \}$. The resulting space of $A$ matrices unconditionally satisfies Theorem \ref{th:scon}, ensuring $(A, B, C, D) \in \Omega$ throughout training. This ensures trajectories generated by the SEGP will be bounded, since any $A$ matrix satisfying (\ref{eq:scon}) has eigenvalues with non-positive real components \cite{hespanha2018linear}. Thus numerical issues due to the term $\e^{At}$ are avoided when computing the mean and covariance. Furthermore, the sets $\mathbb{L}_{\geq 0}^{n}$, $\mathbb{L}_{+}^{n}$ can be considered as unconstrained since any positive element can be expressed as the the magnitude of a scalar plus a small positive value. An analogous case holds for skew-symmetric matrices. Finally, as diagonal elements of triangular matrices correspond to their eigenvalues \cite{petersen2008matrix}, $P$ will always be full rank, ensuring the existence of it's inverse.

\section{Stability Enhanced Gaussian Process Variational Autoencoder} \label{sub:GPVAE}
This section describes the general form of the SEGP-VAE architecture and the unsupervised training objective. Consider the joint distribution between the video and the latent state, where $\vbf := \vbf(T)$, $\y := \y(T)$, $\vbf_{i} := \vbf(t_{i})$, and $\y_{i} := \y(t_{i})$. The joint distribution can be expressed as $P (\vbf, \y) = P (\vbf \vert \y) P_{\psi}(\y)$ where $P_{\psi}(\y)$ denotes the SEGP prior with trainable parameters $\psi$. By the independence assumption (iv) in Section \ref{sub:ps}, the conditional distribution can be expressed as the product $P (\vbf \vert \y) = \prod_{i=1}^{N} P ( \vbf_{i} \vert \y_{i} )$. Finally, in this work we make the simplifying assumption that the pixel values are black and white (i.e., $\vbf_{i} \in \{ 0,1 \}^{d^2}$ ). This allows the Bernoulli to be used to model the likelihood of the video, resulting in
\begin{align}
P (\vbf, \y) &= \prod_{i=1}^{N} \B \big( \vbf_{i} \vert \p_{\theta}( \y_{i} ) \big) P_{\psi}(\y) \label{eq:VAE_post3}
\end{align}
The Bernoulli is parametrised by $\p_{\theta}: \R^{m} \rightarrow [0, 1]^{d^{2}}$ which represents the probability of each pixel being white. This function is implemented by the decoder of the VAE, with trainable parameters $\theta$. Any universal function approximator may be used for the decoder providing each output is restricted to $[0,1]$. For example, an MLP with sigmoid output activations. Note that the Bernoulli could easily be replaced by the Beta distribution when modelling video with normalised RGB pixel values. In this case, the decoder would map the latent state to the shape parameters.

Due to the Bernoulli distribution (or Beta), there exists no analytical solution for the posterior, $P (\y | \vbf)$; hence, the variational approximation of the posterior proposed in \cite{pearce2020gaussian} is reused
\begin{equation} \label{eq:Var_post}
Q_{\phi, \psi} \big(\y \vert \vbf \big) \hspace{-0.5mm} = \hspace{-0.5mm} \frac{1}{L_{\phi, \psi}(\vbf)} \prod_{i=1}^{N} \underbrace{\No \big( \y_{i} \vert \m_{\phi}(\vbf_{i}), \Sigma_{\phi}(\vbf_{i}) \big)}_{:= q_{\phi}(\y_{i} \vert \vbf_{i})} \hspace{-1mm} P_{\psi}(\y)
\end{equation}
where $L_{\phi, \psi}(\vbf)$ denotes the marginal likelihood. This approximation is simply the true posterior with only the troublesome likelihood replaced by a Gaussian. The mean and covariance functions of the likelihood, $\m_{\phi}: \{0, 1\}^{d^{2}} \rightarrow \R^{m}$, $\Sigma_{\phi}: \{0, 1\}^{d^{2}} \rightarrow \mathbb{D}_{+}^{m}$ respectively, are implemented by the VAE encoder, with trainable parameters $\phi$. Any universal function approximator may be used for the encoder providing the outputs, corresponding to the diagonal elements of $\Sigma_{\phi}$, are guaranteed to be positive. For example, mapping these outputs through an exponential function guarantees this. The SEGP-VAE architecture is depicted in Fig. \ref{fig:GP_VAE}.

We consider how to learn an accurate approximation of the posterior in Section \ref{sub:train}; for now, we just highlight how $Q_{\phi, \psi}(\y \vert \vbf)$ is computed. Consider a single video, $\vbf(T)$, from the dataset. Mapping it through the encoder results in an associated set of latent ``observations'', $\mathcal{L} := \{ \m_{\phi}(\vbf_{i})\}_{i=1}^{N}$, each with noise $\Sigma_{\phi}(\vbf_{i})$. Conditioning the GP prior on these observations yields the analytically tractable posterior (\ref{eq:Var_post}) which could be extracted from a joint Gaussian of the form (\ref{eq:GP_joint}) where $\y(T)$ and $\Sigma$ are replaced by $\m_{\phi}(\vbf) \in \R^{mN}$ and $\Sigma_{\phi}(\vbf) \in \mathbb{D}_{+}^{mN}$, as defined below
\begin{align*}
\vspace{-\mylen}
\m_{\phi}(\vbf) &= 
\setlength\arraycolsep{0pt} 
\begin{bmatrix} 
\mu_{\phi, 1}(\vbf_{1}) & \dots &  \mu_{\phi, 1}(\vbf_{N}) & \dots &  \mu_{\phi, m}(\vbf_{1}) & \dots &  \mu_{\phi, m}(\vbf_{N}) 
\end{bmatrix} \\
\Sigma_{\phi}(\vbf) &= {\small \diag \hspace{-1mm} \big(
\setlength\arraycolsep{0pt} 
\begin{matrix} 
\sigma_{\phi, 1}^{2}(\vbf_{1}) & \dots &  \sigma_{\phi, 1}^{2}(\vbf_{N}) & \dots &  \sigma_{\phi, m}^{2}(\vbf_{1}) & \dots &  \sigma_{\phi, m}^{2}(\vbf_{N}) \end{matrix}
\big)}
\vspace{-\mylen}
\end{align*}
As a result, the mean and covariance of the posterior can be directly computed from the standard GP equations (\ref{eq:GP_pred})-(\ref{eq:GP_pred_cov}).

\subsection{Training}
As this is an unsupervised learning task, the goal was to obtain an accurate reconstruction of the input video, whilst indirectly training the SEGP to accurately model the latent process. To this end, an augmented version of the \emph{evidence lower bound (ELBO)} objective was employed. The ELBO objective is defined as
\begin{multline} \label{eq:ELBO}
   \hspace{-2mm} \mathcal{L}_{ELBO}(\cdot) = \E_{Q_{\phi, \psi}} \Big[ \sum_{i=1}^{N} \log \B \big( \vbf_{i} \vert \p_{\theta}(\y_{i}) \big) \Big] \\
   + \beta \cdot \KL \big( P_{\psi} (\y) \vert \vert Q_{\phi, \psi}(\y \vert \vbf) \big)  
\end{multline} 
where the expectations are with respect to the variational posterior (\ref{eq:Var_post}). The first term is the reconstruction error which indirectly depends on the encoder, $\phi$, and SEGP, $\psi$, as $\y_{i} \sim Q_{\phi, \psi}(\y_{i} \vert \vbf_{i})$ is sampled using the reparameterisation trick \cite{kingma2013auto}. The second term is the analytically tractable Kullback–Leibler (KL) divergence, which regulates how far the variational posterior deviates from the SEGP prior. Typically, the weighting parameter $\beta$ is set to one; however, we treat this as a hyperparameter to be tuned. We augment the ELBO loss by introducing an additional regularisation term on the SEGP parameters, $r(\psi)$, as shown below. This term is also weighted by a tunable hyperparameter, $\lambda$. The choice of regularization term is task dependent and provides additional guidance when training the SEGP.
\begin{equation} \label{eq:aug_ELBO}
    \mathcal{L}(\phi, \psi, \theta; \vbf) = \mathcal{L}_{ELBO}(\phi, \psi, \theta; \vbf) + \lambda r(\psi)
\end{equation}

\section{Case Study: Spiralling Particle} \label{sec:app}
Prior to this section, the general SEGP-VAE model was presented; however, as with all ML approaches, some problem specific choices must be made to achieve strong performance. This section begins by detailing the data which is subsequently used to justify certain choices in the model and training setup. Finally, the empirical results are presented with the code and weights freely available at {\tt \small https://github.com/pnnl/COPIP}.

\subsection{Data} \label{sec:exp}
A dynamical system of the form (\ref{eq:LTI}), where $\x = [ r \hspace{2mm} \theta ]'$, represents the position of a particle in a plane in terms of its polar coordinates (radius, $r$, and angle from the positive x-axis, $\theta$). The following matrices were used
\begin{align*}
\hat{A} &= 
\begin{bmatrix}
-0.6 & 0 \\
0 & 0
\end{bmatrix}
&
\hat{B} &= 
\begin{bmatrix}
0 \\
1
\end{bmatrix}
& 
\hat{C} &= I
&
\hat{D} &= 0 
\end{align*}
which form an observable but uncontrollable system. The external input $u(t) \sim \GP \big(\mu_{u}(t), k_{u}(t, t') \big)$ was determined by $\mu_{u}(t) = 0.4 \pi t$ and $k_{u}(t, t') = \exp \big( -0.5(t-t')^{2} \big)$ and the initial conditions were sampled from $\x(0) \sim \No(\m_{x0}, \sigma_{x0}^{2} I )$ where $\m_{x0} = [ 1.5 \hspace{2mm} 0]'$ and $\sigma_{x0} = 0.2$. Euler integration was used to simulate the trajectories over a $3s$ period with step size $1 \times 10^{-2}$. The trajectories were discretised with a sample period of $T_{s} = 12 \times 10^{-2}$ (N = 25) and additive measurement noise was applied to the discretised trajectories with $\Sigma = \sigma^{2} I$ and $\sigma^{2} = 1 \times 10^{-3}$. The noisy discretised trajectories were transformed to Cartesian coordinates, rescaled to pixel indices and rendered as a ball, with a radius of 2 pixels, onto a $H \times W $ binary pixel canvas, where $H = W = 40$. In total, $N_{v} = 40,000$ videos were generated by sampling $N_{v}$ external inputs and $1$ initial condition for each. Example videos are shown in Fig. \ref{fig:frame}, where the images at each time step have been overlaid to show the full trajectory.

\subsection{SEGP-VAE} \label{sec:model}

\subsubsection{Encoder}
Since we are working with images, the encoder of the SEGP-VAE was implemented by a CNN-based architecture. This consisted of three convolutional layers, with the first two having stride 2, each followed by Group Normalization and ReLU activations, and the last one having stride 1. This produced a set of latent feature maps $F \in \mathbb{R}^{K \times H \times W}$, where $K=1$ was the chosen number of key points. Due to the rotational dynamics in the data, a two-step spatial-to-polar mapping was implemented:

(i) For each channel, $k$, a \textit{Spatial Softmax} layer \cite{levine2016end} mapped the latent features, $w$, to specific image coordinates $(x_k, y_k)$. The expected coordinate in the normalized image space was computed as
\begin{equation*}
    \mathbb{E}[v_k] = \sum_{h=1}^H \sum_{w=1}^W \frac{\exp(F_{k,h,w} / \tau)}{\sum_{i,j} \exp(F_{k,i,j} / \tau)} \cdot \text{pos}_v(w)
\end{equation*}
where $v_{k} \in \{x_{k},y_{k}\}$, $\tau=1.0$ was the chosen value of the temperature parameter controlling the sharpness of the attention, and $\text{pos}_v(w)$ denotes the normalized horizontal or vertical position. This operation is fully differentiable.

(ii) The Cartesian coordinates were mapped to polar coordinates $(r, \theta)$ via the $\text{atan2}$ function. However, the inherent $2\pi$-discontinuity of the phase component posed a significant challenge for gradient-based optimisation in temporal sequences. To resolve this, we implemented a temporal unwrapping trick. Given a sequence of raw angles $\{\theta_t\}_{t=1}^N$, we computed the unwrapped sequence $\{\theta'_t\}$ such that
\begin{equation*}
    \theta'_t = \theta_t + 2\pi \cdot \sum_{i=2}^t \text{round} \left( \frac{\theta_{i-1} - \theta_i}{2\pi} \right)
\end{equation*}
By eliminating artificial jumps at $\pm\pi$, the latent state $\mathbf{z}_t = [r_t, \theta'_t]^\top$ evolves in a continuous space. Finally, the encoder mean was obtained by applying an affine transformation to the polar coordinates with learnable scale and bias parameters. The encoder variance was derived from learnable log-variance parameters clamped for numerical stability.

\subsubsection{Decoder}
The decoder was implemented by a simple $2$-layer MLP. Batch norm was applied to the $500$ hidden neurons, followed by ReLU activations. This was followed by a linear output layer, parametrising the Bernoulli likelihood.

\subsubsection{SEGP}
The SEGP was implemented according to Theorem \ref{th:LTI_GP}, with the $A$ matrix parametrised according to Theorem \ref{th:scon_param}. To make the problem tractable, the matrices $B, C, D$ were assumed to be known; hence $\psi = \{V_{1}, V_{2}, V_{3}\}$. This corresponds to having information of the input-output structure of the LTI system, but no knowledge of the internal dynamics. Whilst this isn't a general requirement, structural identifiability was an issue for this particular task as the system is not controllable; hence there is no unique mapping between the transfer function (TF) and state space model.

\subsection{Training Setup} \label{sub:train}

\subsubsection{Objective Function}
To ensure that the learned latent dynamical system remains physically interpretable, we chose to include $L_1$ regularization on the state matrix $A$. Specifically, this led to $r(\psi) = \|A\|_1$ in \eqref{eq:aug_ELBO} where $\|A\|_1 = \sum_{i,j} |A_{i,j}|$. The L1 regularization promotes sparsity and elements of lower magnitude. This acts as a complexity control mechanism, encouraging simpler and more interpretable dynamics. We also employed a linear scheduler for the weight of the $L1$ penalty term. During training, $\lambda$ increased from 0.025 to 0.3 to ensure training initially focused on minimising the reconstruction error. These hyperparameters were not rigorously tuned, $\beta$ was simply chosen to prevent the reconstruction term from dominating the loss. 

\subsubsection{Feature Scaling}
As shown in Fig. \ref{fig:learn}, the angular position can be much larger than the radius. Without intervention, this feature would naturally dominate the KL divergence. To ensure both features were treated equally during training, the KL divergence was computed in the normalised space. For the prior and posterior, this involved normalising each element of the mean function and covariance matrix by the corresponding largest radius or angular position from the training dataset.

\subsubsection{Optimizer and Parameter Initialisation}
All model parameters including encoder, decoder and SEGP were trained jointly using the AdamW optimizer with a learning rate of $5 \times 10^{-3}$ and weight decay $10^{-5}$. Kaiming initialisation was used for the encoder and decoder, whilst the $A$ matrix \eqref{eq:scon_param} was initialised with $V_1 = I$, $V_3=0$, and $V_2$ sampled from a zero-mean Gaussian with standard deviation $10^{-3}$. These choices ensured the initialised LTI system was mildly stable. 

\subsection{Empirical Results}
To separate the roles of the stability parametrisation, the encoder, and the regulariser, we trained and studied the SEGP-VAE along with three ablations: the regularisation term was omitted from the loss i.e., $\lambda=0$ (NREG); the encoder was replaced by a vanilla MLP as used in \cite{beckers2023physics} (MLP2); and the stability parametrisation was omitted while retaining the optimised encoder (PEGP). 

%
\begin{figure}[t!]
    \centering
    \includegraphics[width=\linewidth]{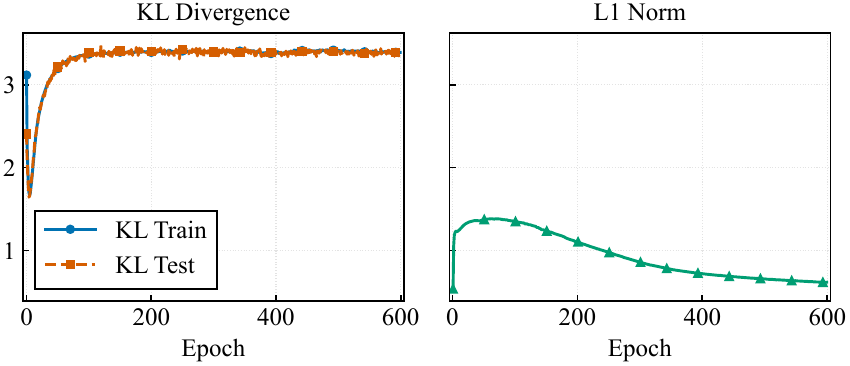}
    \vspace{-2\mylen}
    \caption{KL divergence (unweighted) and L1 norm during training.}
    \label{fig:kl}
    \vspace{-2\mylen}
\end{figure}

\subsubsection{Training}
We report the results after training for 600 epochs, with both training and testing curves monitored to evaluate effectiveness. Training of the full SEGP-VAE converged without evidence of overfitting, and the non-zero KL divergence (Fig. \ref{fig:kl}) indicated that posterior collapse was avoided. The $L_1$ penalty continued to decrease during training, consistent with its role in shaping the recovered dynamics (Fig. \ref{fig:kl}). As expected, training the PEGP-VAE frequently encountered numerical issues due to the possibility of learning a non-Hurwitz $A$ matrix. When computing the mean and covariance of the prior, this often led to numerical values tending towards infinity, thus preventing derivatives from being calculated. Both the MLP2-VAE and NREG-VAE trained smoothly. 

\subsubsection{Learnt LTI systems}
The A matrices recovered from each experiment were (to 2dp)
\begin{align*}
A_{\mathrm{SEGP}} &= 
\begin{bmatrix}
-0.62 & \; \; \, 0.00 \\
-0.00 & -0.00
\end{bmatrix}
&
A_{\mathrm{PEGP}} &= 
\begin{bmatrix}
0.00 & \; \; \, 0.00 \\
-0.17 & \; \; \, 0.00
\end{bmatrix}
\\
A_{\mathrm{MLP2}} &= 
\begin{bmatrix}
-0.00 & \; \; \, 0.00 \\
-0.18 & -0.00
\end{bmatrix}
&
A_{\mathrm{NREG}} &= 
\begin{bmatrix}
-0.21 & \; \; \, 0.04 \\
2.42 & -0.63
\end{bmatrix}
\end{align*}
Importantly, $(A_{\mathrm{PEGP}})_{11} = 8.8 \times 10^{-4}$ and $(A_{\mathrm{PEGP}})_{22} = 1.1 \times 10^{-4}$ which implies the PEGP-VAE recovered a marginally unstable LTI system. The semi-contracting parametrisation ensured stable LTI systems were recovered in all other cases; however, without the regularisation and optimised encoder, the true co-ordinate system was poorly approximated. When mapping the learnt LTI systems to their corresponding TFs, only the NREG-VAE returned a poor approximation of the true TF. This implies the input-output properties of the underlying dynamical system which generated the data were well matched provided the regularisation term was included in the loss. 

Table \ref{tab:L2} includes the $L_2$ distances between the learnt and true $A$ matrices. This measure indicates that the worst case directional distortion due to the error in $A$ was minimal when the full SEGP-VAE was used. The larger values for the other models highlight their poor approximation of the internal dynamics and emphasise the system identifiability problem.
\begin{table}[h]
    \centering
    \begin{tabular}{c|cccc}
         & $\mathrm{SEGP}$ & $\mathrm{PEGP}$ & $\mathrm{MLP2}$ & $\mathrm{NREG}$ \\
        \hline
        $2$-norm & 0.02 & 0.63 & 0.63 & 2.53
    \end{tabular}
    \caption{$\|\hat{A} - A_{i}\|_{2}$ where $i \in \{\mathrm{SEGP},  \mathrm{PEGP}, \mathrm{MLP2}, \mathrm{NREG}\}$.}
    \label{tab:L2}
    \vspace{-2\mylen}
\end{table}
As a result, the covariance matrix corresponding to the SEGP-VAE (Fig. \ref{fig:cov}) is near identical to that of the true LTI system. It exhibits the uncertainty in the radius exponentially decaying and the angular position integrating the uncertainty of the input over time. As expected due to the well-approximated TF argument, the PEGP-VAE and MLP2-VAE have near identical covariance matrices as the SEGP-VAE. The covariance matrix of the NREG-VAE is also plotted in Fig. \ref{fig:cov}. It is clear that this identifies the expected pattern, but without the regularisation term, the scale is poorly approximated.
\begin{figure}[t!]
    \centering
    \includegraphics[width=\linewidth]{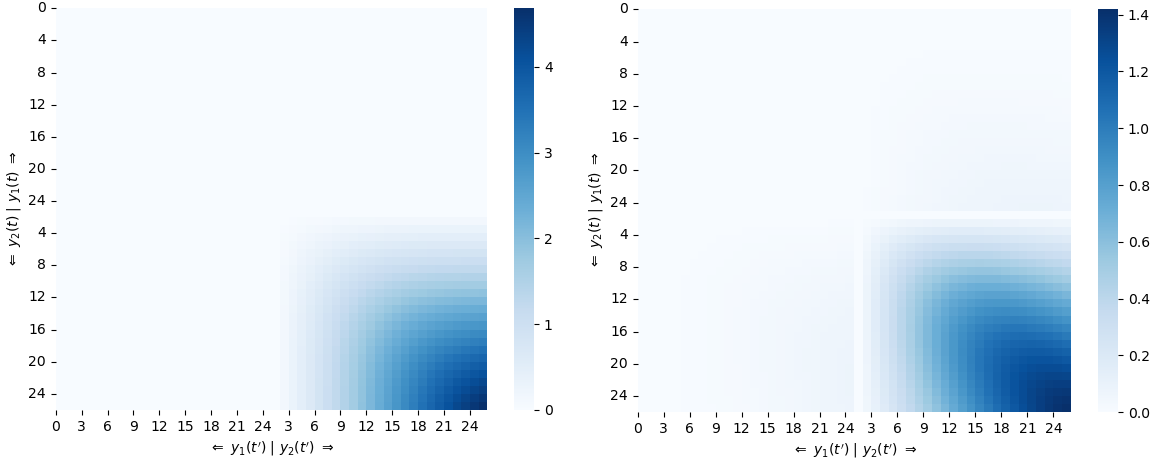}
    \caption{Learnt covariance matrices of SEGP-VAE prior (left) \& NREG-VAE prior (right). Top left block is the covariance of the radius and bottom right block is the covariance of the angular position. Top right and bottom left are the intra-dimensional covariances.}
    \label{fig:cov}
    \vspace{-2\mylen}
\end{figure}
    
\subsubsection{Posterior versus prior}
Fig. \ref{fig:postvprior} and Fig. \ref{fig:postvprior2} show the average absolute error (between the mean and the angle evolutions from the test set) and the variance over time. Only the angular position is considered since there is such little variance in the radius data. 

The specially designed encoder is important for recovering the correct co-ordinate system of the latent state, even when the TF is already approximately correct. The MLP2-VAE and PEGP-VAE ablations have approximately equivalent TFs, yet their posterior angle errors are substantially larger than those of the SEGP-VAE. By mapping image features to spatial coordinates and then to an unwrapped polar representation, the encoder supplies a continuous, physically aligned latent observation. The sharp reduction in posterior error and variance relative to the prior in Fig. \ref{fig:postvprior} therefore shows that it has learned an effective Bayesian update from the video.

\begin{figure}[t!]
    \centering
    \includegraphics[width=\linewidth]{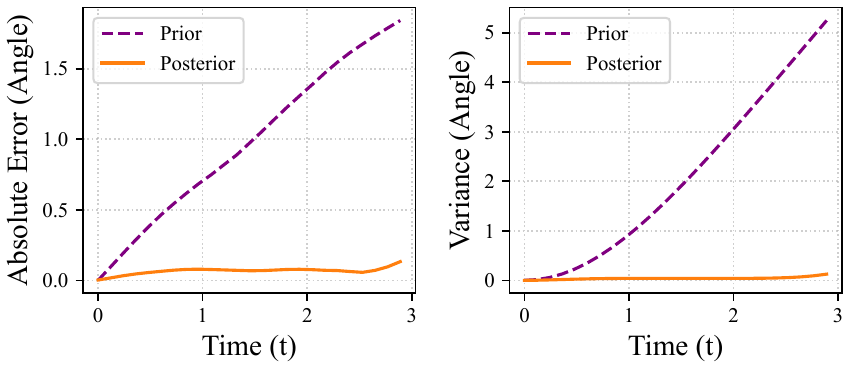}
    \vspace{-3\mylen}
    \caption{Absolute error between the learnt SEGP mean and the angular position trajectory data, averaged over the test set (left). Variance of the learnt SEGP for the angular position, averaged over the test set (right).} 
    \label{fig:postvprior}
    \vspace{-\mylen}
\end{figure}

\begin{figure}[t!]
    \centering
    \includegraphics[width=\linewidth]{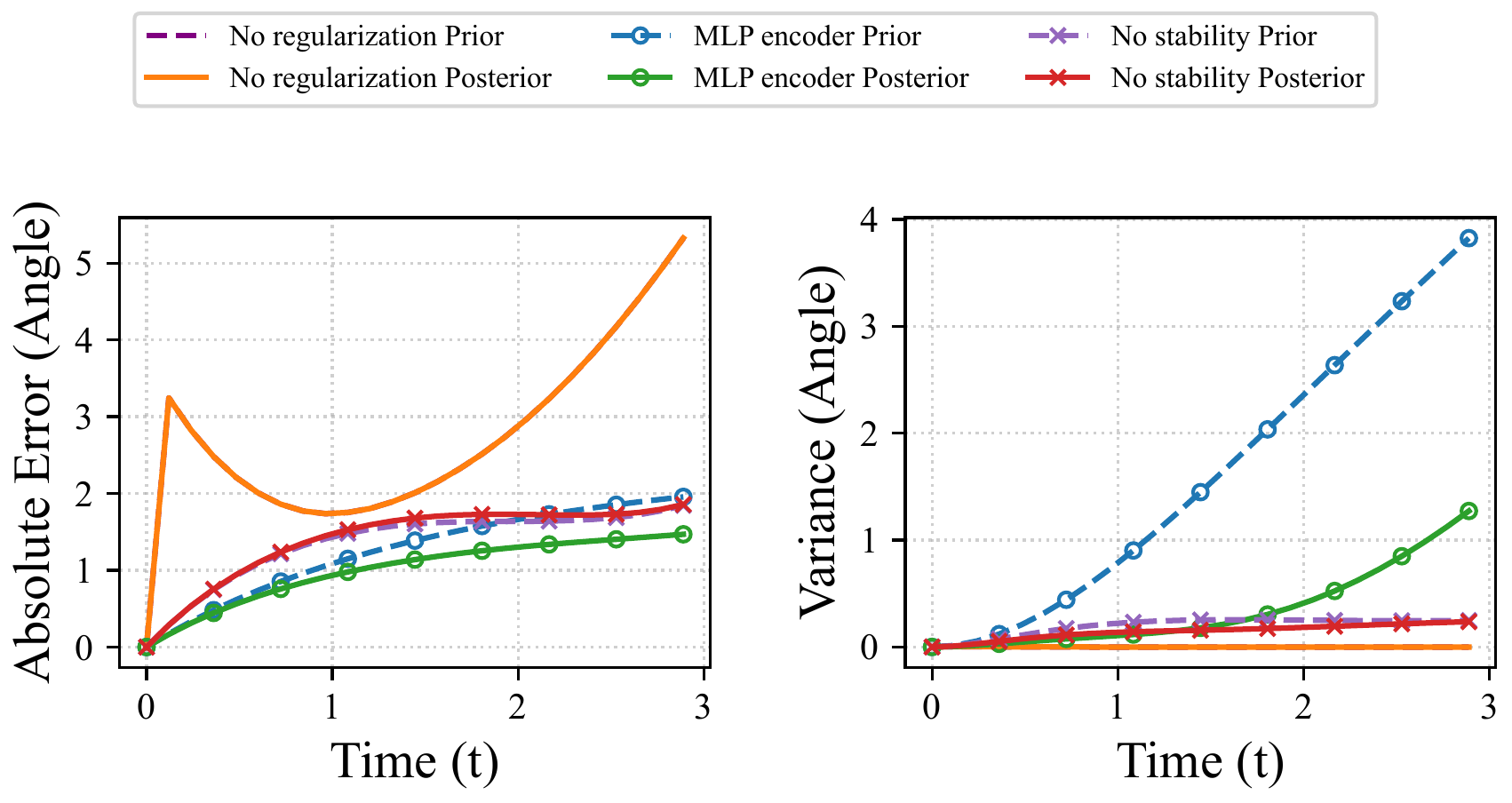}
    \vspace{-3\mylen}
    \caption{Absolute error between the mean of learnt models and the angular position trajectory data, averaged over the test set (left). The variance of learnt models for the angular position, averaged over the test set (right).} 
    \label{fig:postvprior2}
    \vspace{-\mylen}
\end{figure}

\begin{figure}[t!]
    \centering
    \includegraphics[width=\linewidth]{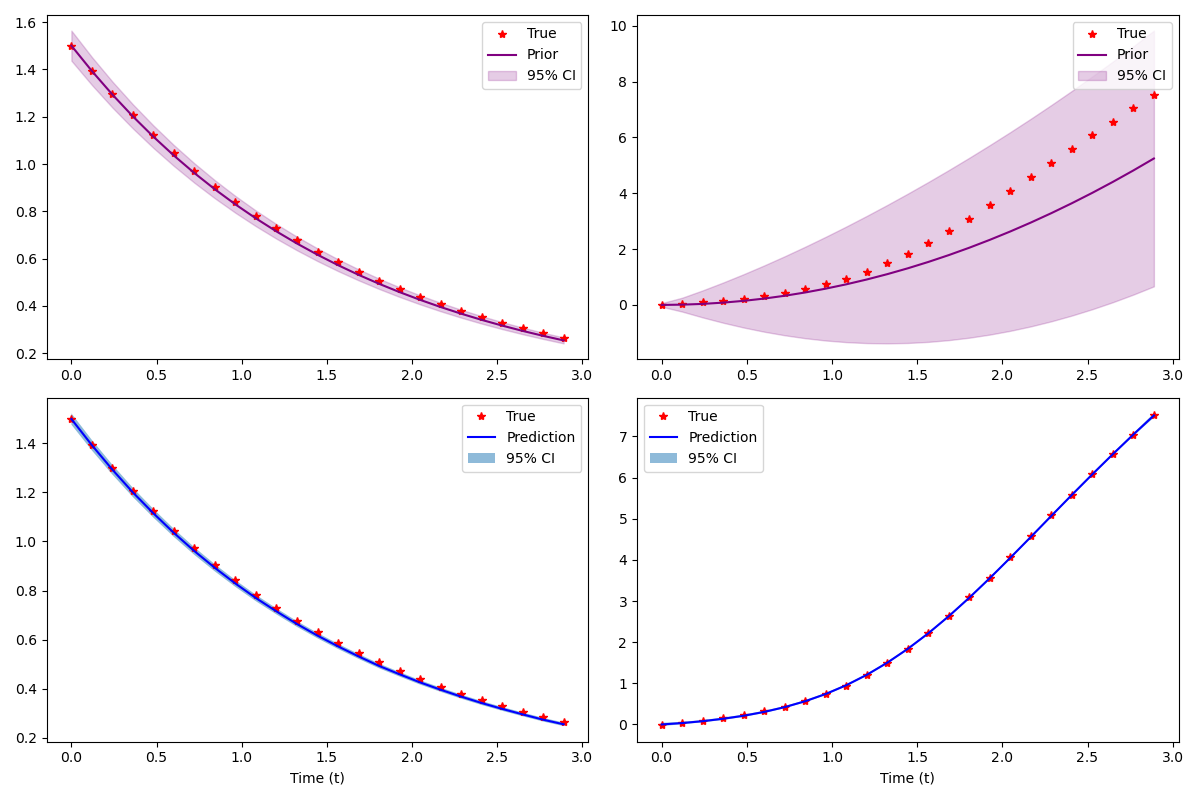}
    \vspace{-3.5\mylen}
    \caption{Randomly sampled latent trajectory from test set, plotted against the SEGP-VAE prior (top) and posterior (bottom). Radius plotted on left and angular position on right.}
    \label{fig:learn}
    \vspace{-2\mylen}
\end{figure}

Taken together, these results show why all three ingredients are needed for this example. Because the true system is uncontrollable, multiple internal state-space realisations can share nearly the same TF; the problem is therefore structurally unidentifiable from input-output behaviour alone. The stability parametrisation selects admissible dynamics, the $L_1$ penalty favours a lower-complexity and lower-magnitude realisation, and the application-specific encoder provides the state-aligned information needed to distinguish useful latent predictions. Their combination, rather than TF accuracy alone, gives the full SEGP-VAE its accurate and interpretable latent dynamics.

\subsubsection{Prediction}
Figures \ref{fig:learn} - \ref{fig:learn4} plot examples of the priors and posterior predictions against trajectory data randomly sampled from the test set. Only the SEGP-VAE (Fig. \ref{fig:learn}) is able to accurately capture the uncertainty and mean of both the radius and the angle. In fact, the radius is poorly predicted in all other cases. Finally, Fig. \ref{fig:frame} shows the reconstructed frames of the SEGP-VAE accurately reproduce the rotational motion present in the ground truth sequences. Hence, this combination of model and loss function is able to connect the underlying dynamics of the latent process with the observed video recordings. 

\begin{figure}[t!]
    \centering
    \includegraphics[width=\linewidth]{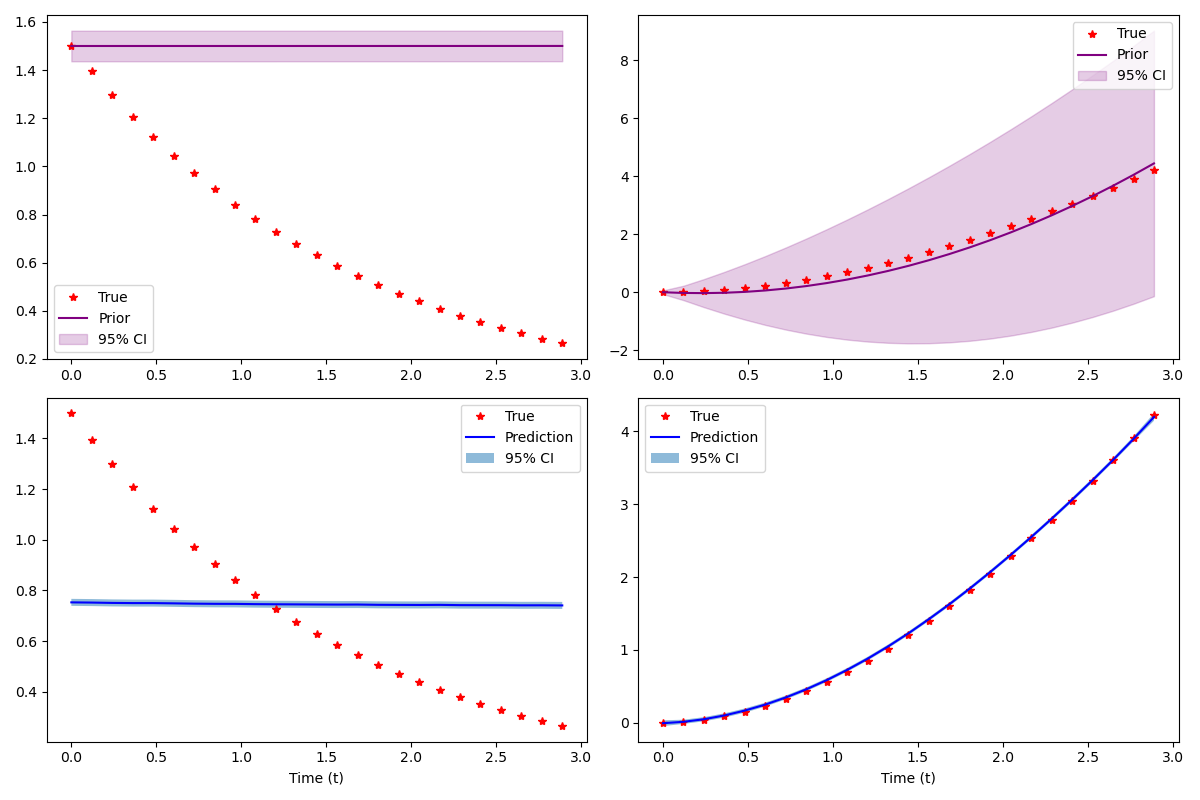}
    \vspace{-3.5\mylen}
    \caption{Randomly sampled latent trajectory from test set, plotted against the MLP2-VAE prior (top) and posterior (bottom). Radius plotted on left and angular position on right.}
    \label{fig:learn2}
    \vspace{-1.5\mylen}
\end{figure}

\begin{figure}[t!]
    \centering
    \includegraphics[width=\linewidth]{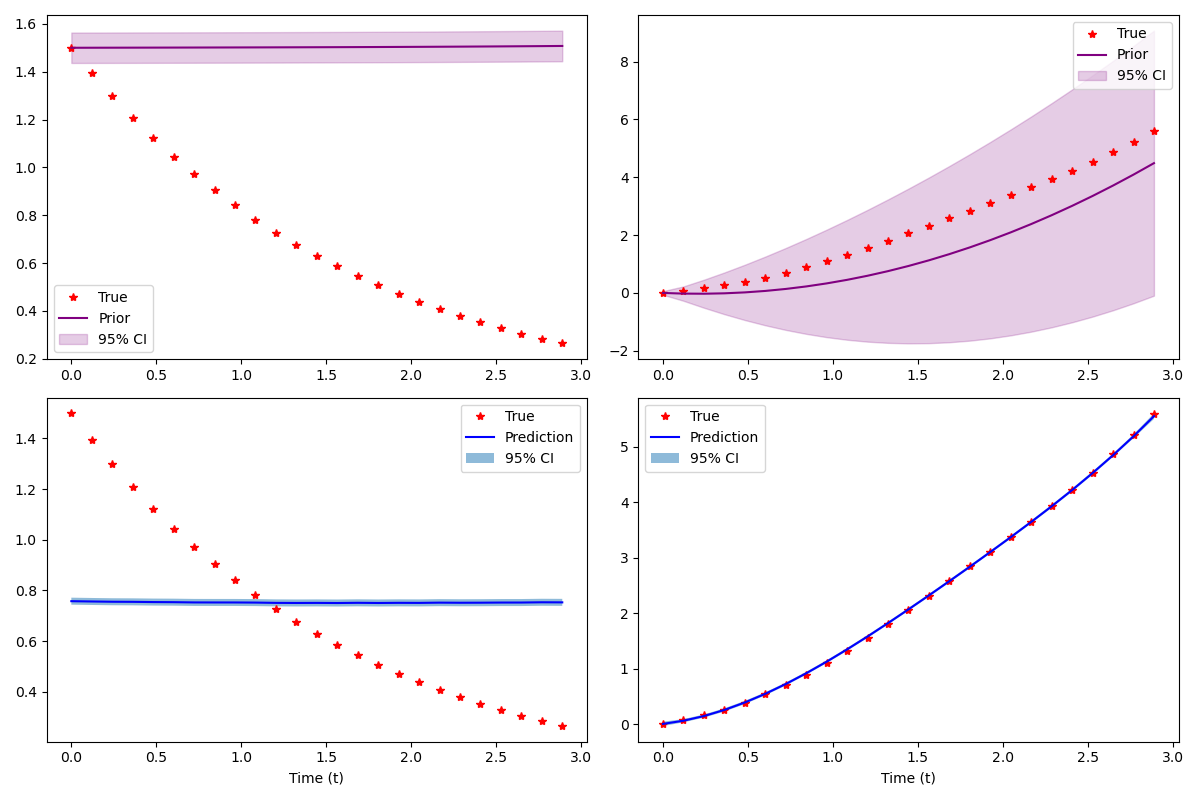}
    \vspace{-3.5\mylen}
    \caption{Randomly sampled latent trajectory from test set, plotted against the learnt PEGP-VAE prior (top) and posterior (bottom). Radius plotted on left and angular position on right.}
    \label{fig:learn3}
    \vspace{-1.5\mylen}
\end{figure}

\begin{figure}[t!]
    \centering
    \includegraphics[width=\linewidth]{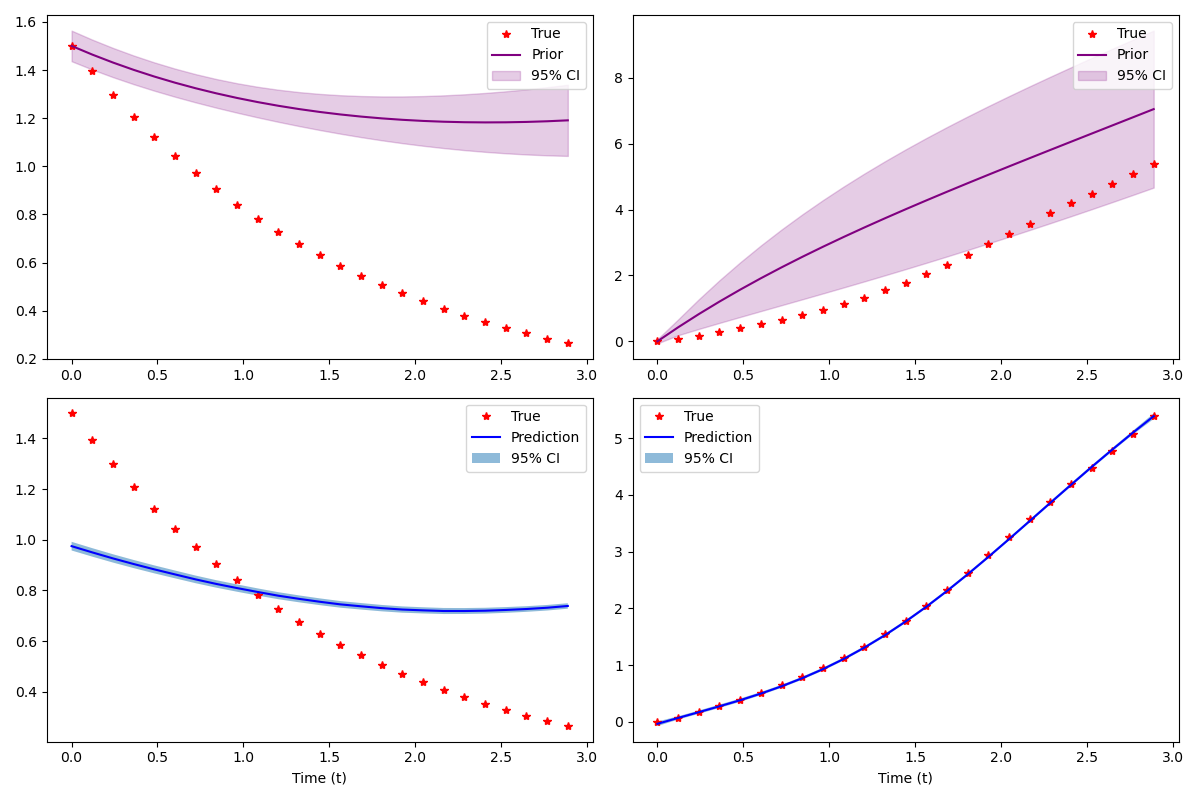}
    \vspace{-3.5\mylen}
    \caption{Randomly sampled latent trajectory from test set, plotted against the learnt NREG-VAE prior (top) and posterior (bottom). Radius plotted on left and angular position on right.}
    \label{fig:learn4}
    \vspace{-2\mylen}
\end{figure}

\addtolength{\textheight}{-0.9cm}

\section{Conclusion} \label{sec:con}
A novel stability-enhanced Gaussian process (SEGP) was derived from a semi-contracting LTI system and incorporated into a VAE, providing an interpretable probabilistic model for learning latent dynamics from video. In the case study, the stability parametrisation prevented the numerical issues encountered by the unconstrained PEGP-VAE, while $L_1$ regularisation favoured a lower-magnitude state matrix. The application-specific encoder was also essential. Although several ablations recovered a similar transfer function, only the complete SEGP-VAE accurately recovered the latent coordinates and their uncertainty. These results highlight the practical importance of combining stability, regularisation, and a state-aligned encoder when the underlying system is uncontrollable and therefore structurally unidentifiable from input-output behaviour alone. Future work will address the need for prior input-output structure and extend the approach to nonlinear systems through learned Koopman observables, with sparse GP approximations offering a route to control computational cost.

\begin{figure}[t!]
    \centering
    \includegraphics[width=\linewidth]{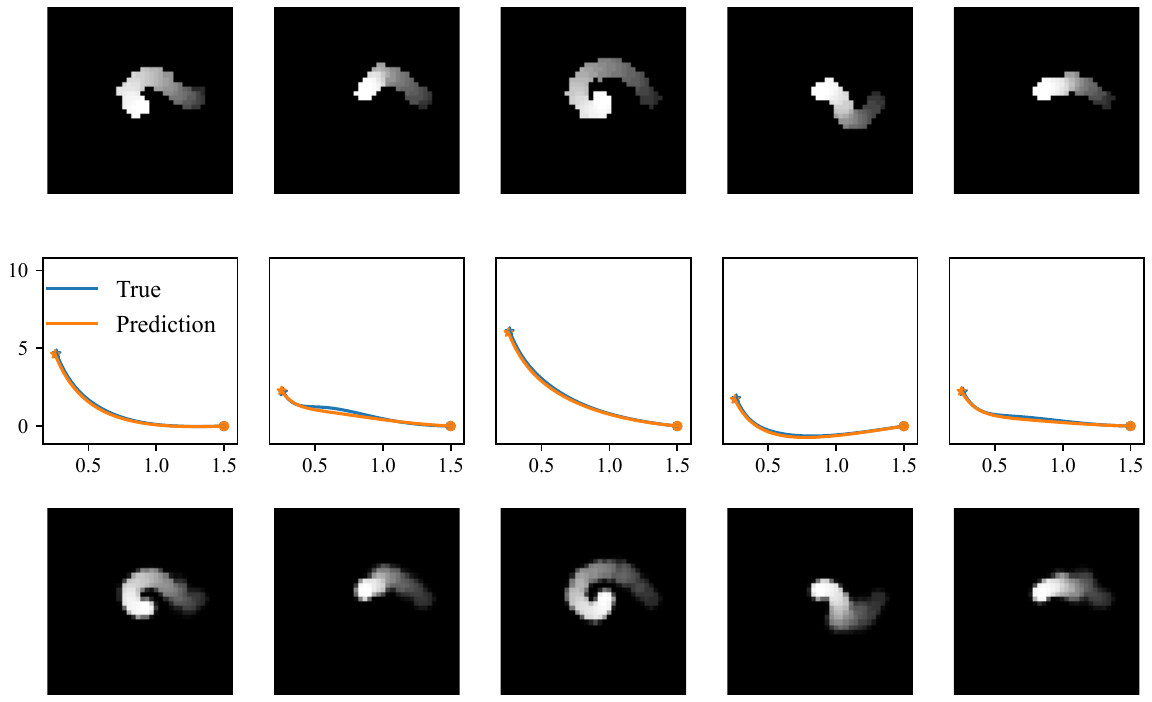}
    \vspace{-2.5\mylen}
    \caption{Randomly sampled videos from test set (top); corresponding latent trajectories (radius along x-axis, angular position along y-axis) and the SEGP posterior mean (middle); SEGP-VAE reconstructed video (bottom).}
    \label{fig:frame}
    \vspace{-2.5\mylen}
\end{figure}

\appendix

\section{Proof of Theorem \ref{th:scon}}
The fundamental result of contraction analysis is presented in \cite[Section 3]{lohmiller1998contraction}. For (\ref{eq:LTI}) to be semi-contracting, the distance between any two trajectories, in a Riemann space, must never increase. This distance is denoted by $V(\delta x)$, where $\delta x$ denotes the virtual displacement between two trajectories and $P \in \mathcal{S}_{+}^{n}$. The virtual displacement is governed by $\dot{\delta x}$. 
\begin{align} \label{eq:dx}
V(\delta x) &= \delta x^{\top} P \delta x &
\dot{\delta x} &= A \delta x
\end{align}

\emph{Proof of Theorem \ref{th:scon}}: Consider $\dot{V}(\delta x)$
\begin{align*}
\dot{V}(\delta x) = 2 \delta x^{\top} P \dot{\delta x} 
= 2 \delta x^{\top} P A \delta x
= \delta x^{\top} ( P A + A^{\top} P ) \delta x
\end{align*}
If there exists a matrix $P \in \mathcal{S}_{+}^{n}$ satisfying \eqref{eq:scon}, then $\dot{V}(\delta x) \leq 0 \hspace{1mm} \forall \hspace{1mm} \delta x$. Hence, the system is semi-contracting. $\hfill \square$

\bibliography{Bibliography.bib}
\bibliographystyle{IEEEtran}

\end{document}